\documentclass[letterpaper, 10 pt, conference]{ieeeconf}

\usepackage{authblk}
\usepackage{lineno,hyperref}
\usepackage{xcolor}
\usepackage{graphicx}
\usepackage[FIGTOPCAP]{subfigure}
\usepackage{algorithm}
\usepackage{algpseudocode}
\usepackage{amsmath}
\usepackage{amsfonts}
\usepackage{amssymb}
\usepackage{fancyhdr}
\usepackage{cuted}
\usepackage{physics}
\usepackage{mathtools}

\newtheorem{prop}{Proposition}
\newtheorem{lemma}{Lemma}
\newtheorem{remark}{Remark}

\DeclareMathOperator{\diag}{diag}
\DeclareMathOperator{\mrank}{rank}

\IEEEoverridecommandlockouts                              

\title{\LARGE \bf
A Dynamic Mode Decomposition Approach \\ for Decentralized Spectral Clustering of Graphs}
\author{Hongyu Zhu$^{1}$, Stefan Klus$^{2}$, and Tuhin Sahai$^{1}$

\thanks{$^{1}$Raytheon Technologies Research Center, CT/CA, USA \\
        {\tt\small (HongyuAlice.Zhu@rtx.com, Tuhin.Sahai@rtx.com)}}%
\thanks{$^{2}$University of Surrey, UK
        {\tt\small (s.klus@surrey.ac.uk)}}%
}
\begin{document}

\maketitle
\thispagestyle{plain}
\pagestyle{plain}


\begin{abstract}
We propose a novel robust decentralized graph clustering algorithm that is \emph{provably} equivalent to the popular spectral clustering approach. Our proposed method uses the existing wave equation clustering algorithm that is based on propagating waves through the graph~\cite{sahai2010wave,SahaiSperanzonBanaszuk2011}. However, instead of using a fast Fourier transform (FFT) computation at every node, our proposed approach exploits the Koopman operator framework. Specifically, we show that propagating waves in the graph followed by a \emph{local} dynamic mode decomposition (DMD) computation at \emph{every} node is capable of retrieving the eigenvalues and the local eigenvector components of the graph Laplacian, thereby providing local cluster assignments for all nodes. We demonstrate that the DMD computation is more robust than the existing FFT based approach and requires 20 times fewer steps of the wave equation to accurately recover the clustering information and reduces the relative error by orders of magnitude. We demonstrate the decentralized approach on a range of graph clustering problems.
\end{abstract}


\section{Introduction}
Interconnected networks characterized by complex behavior due to interacting agents and subsystems can be represented by weighted graphs, with the weight indicating the strength of the connection between agents or nodes. Graph clustering methods that partition nodes into groups with strong intra-connections but weak inter-connections arises naturally in these networks to assist in decision making and co-ordination~\cite{sahai2010wave,SahaiSperanzonBanaszuk2011}. In particular, clustering can be used to accelerate distributed estimation and optimization in mobile sensor networks~\cite{SahaiSperanzonBanaszuk2011,ghiasi2002optimal}. The problem of efficiently computing clusters in large graphs also arises in anthropology~\cite{kottak2015cultural}, structural decomposition of complex systems~\cite{mezic2019spectral}, biology~\cite{speer2005functional,girvan2002community}, distributed computation of differential equations~\cite{klus2011efficient}, uncertainty quantification~\cite{surana2012iterative}, and the design of interconnected dynamical systems~\cite{banaszuk2011scalable}, 
to name a few.

Among the clustering methods, spectral clustering has emerged as a powerful tool of choice for graph decomposition \cite{von2007tutorial}. The method assigns nodes to clusters based on the signs of the elements of the eigenvectors of the graph Laplacian corresponding to increasing eigenvalues. For large scale dynamic networks, spectral clustering methods are either infeasible or suffer from slow convergence. A fast decentralized algorithm to cluster graphs was developed by propagating waves through the graph and retrieving the clustering information by a local fast Fourier transform (FFT) \cite{sahai2010wave, SahaiSperanzonBanaszuk2011}. However, the FFT assumes that all frequencies are rationally related, which is not true, in general, for the waves propagating through the graph~\cite{nussbaumer1981fast}. In the wave equation approach outlined in~\cite{sahai2010wave,SahaiSperanzonBanaszuk2011}, the frequencies depend on the eigenvalues of the graph Laplacian, which may be irrationally related. This typically leads to \emph{spectral leakage and noise} that results in the incorrect computation of cluster assignments for some nodes in the graph. Although the incorrect estimation of frequencies at some of the nodes can be mitigated using a voting scheme~\cite{DandachVoting2012}, the approach requires the correct spectral estimation by a majority of neighbors for every node. Given that the accurate estimation of \emph{critical frequencies} at any specific node depends on the eigenvalues of the Laplacian and number of updates of the wave equation, the approach may run into numerical issues.

In this work, we propose a local Dynamic Mode Decomposition (DMD) approach to retrieve the correct clustering information. DMD is an algorithm for computing Koopman modes for a dictionary containing linear basis functions \cite{tu2013dynamic, KNKWKSN18} and is more \emph{robust} than FFT based schemes. In particular, the approach does not assume that the eigenvalues are rationally related and does not suffer from spectral leakage. We build a DMD based distributed graph clustering algorithm that scales to large graphs and sensor networks, reduces the relative error by orders of magnitude, and significantly reduces the required number of wave equation time steps.

The remainder of this work is organized as follows: In Section~\ref{sec:decentralization}, we present the decentralized algorithm for graph clustering based on the dynamic mode decomposition. In Section~\ref{sec:numerics}, we present numerical results for a few graphs. In Section~\ref{sec:comparison}, we compare our new DMD approach with the existing FFT based method for decentralized clustering. Finally, conclusions are drawn in Section~\ref{sec:conclusion}. 

\section{Decentralized approach}\label{sec:decentralization}

\subsection{Spectral clustering}

Let $\mathcal{G}= (V, E)$ be a graph with vertex set $V=\{1,...,N\}$ and edge set $E \subseteq V \times V$, where a weight $\mathbf{W}_{ij}>0$ is associated with each edge $(i,j)\in E$ and $\mathbf{W}$ is the $N \times N$ weighted adjacency matrix of $\mathcal{G}$. Here $\mathbf{W}_{ij}=0$ if and only if $(i, j) \notin E$. The (normalized) graph Laplacian is defined as
\begin{equation}\label{eq:laplacian}
\mathbf{L}_{ij} = \left\{
\begin{array}{ll}
1, & \text{if } i=j, \\
-\mathbf{W}_{ij}/\sum_{l=1}^N\mathbf{W}_{il}, & \text{if } (i,j)\in E, \\
0,  & \text{otherwise}.
\end{array}
\right.
\end{equation}

Note that in this work we only consider undirected graphs, for which the Laplacian is symmetric and its eigenvalues are real. Eigenvalues of $\mathbf{L}$ can be ordered as $0=\lambda_1 \leq \lambda_2 \leq \dots \leq \lambda_N$ with associated eigenvectors $\mathbf{v}^{(1)}, \mathbf{v}^{(2)}, \dots, \mathbf{v}^{(N)}$, where $\mathbf{v}^{(1)}=\mathbf{1}=[1, 1,...,1]^T$, see \cite{von2007tutorial}. We assume $\lambda_1 < \lambda_2$, i.e., the graph does not have disconnected clusters. Spectral clustering divides the graph $\mathcal{G}$ into two clusters using the signs of the entries of the second eigenvector $\mathbf{v}^{(2)}$. More than two clusters can be computed using the signs of the entries of higher eigenvectors or $k$-means clustering, see~\cite{von2007tutorial} for details. Note that the number of clusters can be computed based on spectral gap heuristics.

\subsection{Wave equation based clustering method}
As in \cite{sahai2010wave,SahaiSperanzonBanaszuk2011}, we consider the wave equation,
\begin{equation}\label{eq:wave}
\frac{\partial^2 u}{\partial t^2} = c^2\Delta u.
\end{equation}
The discretized wave equation on a graph is given by,
\begin{equation}\label{eq:discrete_wave}
\mathbf{u}_i(t) = 2\mathbf{u}_i(t - 1) - \mathbf{u}_i(t - 2) - c^2 \sum_{j\in\mathcal{N}(i)} \mathbf{L}_{ij}\mathbf{u}_j(t-1),
\end{equation}
where $\mathcal{N}(i)$ is the set of neighbors of node $i$ including node $i$ itself \cite{friedman2004wave}. To update $\mathbf{u}_i$, one needs only the previous value of $\mathbf{u}_j$ at the neighboring nodes and the connecting edge weights.

As in Proposition 3.1 in \cite{SahaiSperanzonBanaszuk2011}, we can write \eqref{eq:discrete_wave} in matrix form as
\begin{equation}\label{eq:matrix_wave}
\underbrace{\left(\begin{array}{c}
\mathbf{u}(t) \\
\mathbf{u}(t-1)
\end{array}
\right)}_{\mathbf{z}(t)} = 
\underbrace{\left(\begin{array}{cc}
2\mathbf{I} - c^2\mathbf{L} & -\mathbf{I}\\
\mathbf{I}  & \mathbf{0}
\end{array}
\right)}_{\mathbf{M}}
\underbrace{\left(\begin{array}{c}
\mathbf{u}(t-1) \\
\mathbf{u}(t-2)
\end{array}
\right)}_{\mathbf{z}(t-1)},
\end{equation}
and the eigenvalues of $\mathbf{M}$ are,
\begin{equation}\label{eq:eigenvalues_M}
\alpha_{j_{1,2}}=\frac{2-c^2\lambda_j}{2} \pm \frac{c}{2}\sqrt{c^2\lambda_j^2-4\lambda_j},
\end{equation}
where the Laplacian eigenvalues $0 \leq\lambda_j \leq 2$ for all graphs \cite{SahaiSperanzonBanaszuk2011}. If $\mathbf{u}(-1)=\mathbf{u}(0)$ and $0<c<\sqrt{2}$, the eigenvalues $\alpha_j$ have absolute value equal to one, without repetition at $1$ and $-1$. It is interesting to note that the mapping between $\lambda_j$ and $\alpha_{j_{1,2}}$ is bijective. Also, any spectral gaps in the eigenvalues of the Laplacian map to gaps in the spectra of $\mathbf{M}$. Without loss of generality, we can assume
\begin{equation}\label{eq:alpha}
\alpha_{j_{1,2}} = e^{\pm i\omega_j}.
\end{equation}

\begin{lemma}\label{lemma:wave_solution}
Given the initial condition $\mathbf{u}(-1)=\mathbf{u}(0)$ and $0<c<\sqrt{2}$, the solution of the wave equation \eqref{eq:discrete_wave} can be written as
\begin{equation}\label{eq:u_discrete}
\mathbf{u}(t) = \sum_{j=1}^{N}\mathbf{u}(0)^T\mathbf{v}^{(j)}(p_{j}e^{it\omega_j} + q_{j}e^{-it\omega_j})\mathbf{v}^{(j)},
\end{equation}
where $p_{j}=(1+i\tan(\omega_j/2))/2$, $q_{j}=(1-i\tan(\omega_j/2))/2$.
\end{lemma}

\begin{proof}
Using Proposition 3.3 in \cite{SahaiSperanzonBanaszuk2011}, the corresponding eigenvectors to eigenvalues $\alpha_{j_{1,2}}$ are
\begin{equation*}
\mathbf{m}_{1,2}^{(j)} = \mathbf{p}^{(j)} \pm i\mathbf{q}^{(j)},
\end{equation*}
where
\begin{equation*}
\mathbf{p}^{(j)}=
\begin{pmatrix}
\mathbf{v}^{(j)}\cos \omega_j \\
\mathbf{v}^{(j)}
\end{pmatrix}, \,
\mathbf{q}^{(j)}=
\begin{pmatrix}
\mathbf{v}^{(j)}\sin \omega_j \\
0
\end{pmatrix},
\end{equation*}
and
$\mathbf{v}^{(j)}$ are orthonormal eigenvectors of $L$. 
The solution of \eqref{eq:matrix_wave} is
\begin{align*}
\mathbf{z}(t) = & \sum_{j=1}^{N}
C_{j_1}\left[\mathbf{p}^{(j)}\cos(t\omega_j) - \mathbf{q}^{(j)}\sin(t\omega_j)\right] \\
& ~\;+ C_{j_2}\left[\mathbf{p}^{(j)}\sin(t\omega_j) + \mathbf{q}^{(j)}\cos(t\omega_j)\right],
\end{align*}
as given in Proposition 3.3 in \cite{SahaiSperanzonBanaszuk2011}.
Thus,
\begin{equation}
\mathbf{u}(t) = \sum_{j=1}^{N} \left[C_{j_1}\cos((t+1)\omega_j) + C_{j_2}\sin((t+1)\omega_j) \right]\mathbf{v}^{(j)}. \nonumber
\end{equation}
Using $\mathbf{u}(-1)=\mathbf{u}(0)$, we obtain
\begin{equation*}
C_{j_1} (1-\cos(\omega_j)) = C_{j_2}\sin(\omega_j),
\end{equation*}
\begin{equation*}
C_{j_1}  = \mathbf{u}(0)^T \mathbf{v}^{(j)}, C_{j_2} = \mathbf{u}(0)^T \mathbf{v}^{(j)} \tan\left(\frac{\omega_j}{2}\right).
\end{equation*}
Therefore,
\begin{align*}
\mathbf{u}(t) &= \sum_{j=1}^{N} \mathbf{u}(0)^T \mathbf{v}^{(j)} \left(\cos(t\omega_j) - \tan\left(\frac{\omega_j}{2}\right)\sin(t\omega_j)\right) \mathbf{v}^{(j)} \\ &= \sum_{j=1}^{N}\mathbf{u}(0)^T\mathbf{v}^{(j)}(p_{j}e^{it\omega_j} + q_{j}e^{-it\omega_j})\mathbf{v}^{(j)},
\end{align*}
where
\begin{equation*}
    p_{j}=\frac{1}{2}\left(1+i\tan\left(\frac{\omega_j}{2}\right)\right), q_{j}=\frac{1}{2}\left(1-i\tan\left(\frac{\omega_j}{2}\right)\right).
\end{equation*}
\end{proof}
\begin{remark}
The above derivation corrects a small oversight in corresponding derivations in~\cite{sahai2010wave,SahaiSperanzonBanaszuk2011} by including the missing $\tan(\frac{\omega_j}{2})$ terms. The terms, however, do not impact the wave equation clustering algorithm or its performance. 
\end{remark}

In \cite{sahai2010wave,SahaiSperanzonBanaszuk2011}, the above derivation was exploited for clustering by evolving the wave equation (using nearest neighbor information) followed by a local FFT on $[\mathbf{u}_i(1),...,\mathbf{u}_i(T_{max})]$ at the $i$-th node.   $T_{max}$ was based on the minimum number of cycles to resolve the lowest frequency. Since the FFT coefficients are proportional to the elements of the eigenvectors of $\mathbf{L}$ by the same factor across the nodes, the signs of FFT coefficients provide the cluster assignment. The authors demonstrated that the wave equation based algorithm is orders of magnitude faster than random walk based approaches for graphs with large mixing times. However, FFT  requires  that  all  frequencies be rationally related, which is, in general, not true for arbitrary graphs. Additionally, the FFT algorithm works best when the length of the time series $T_{max}$ is a power of 2. For general $T_{max}$, it patches the time series with zeros to a length of closest power of 2, causing \emph{incorrect} frequencies to emerge. Practically, a threshold in signal magnitude is tuned to filter out the noisy frequencies and spectral leakage. Fig.~\ref{fig:karate_graph} shows the FFT result with $T_{max}=1000$ on the Zachary Karate club graph with $N=34$ nodes \cite{zachary1977information}.  The FFT plot~\ref{fig:karate_graph}(b) compares signal magnitudes at different nodes where, at each node, the maximum frequency magnitude is scaled to 1 and the threshold is set to be $1\%$. A greater $T_{max}$ than the theoretical value is usually necessary to avoid spectral leakage.

DMD~\cite{tu2013dynamic}, on the other hand, is a more robust algorithm that does not require the eigenvalues to be rationally related. It also typically requires fewer time steps to resolve the correct frequencies. We now explore the use of DMD in place of the standard FFT algorithm in the context of decentralized graph clustering.

\begin{figure}[h]
\centering
\subfigure[Karate graph]{\includegraphics[width=.52\linewidth]{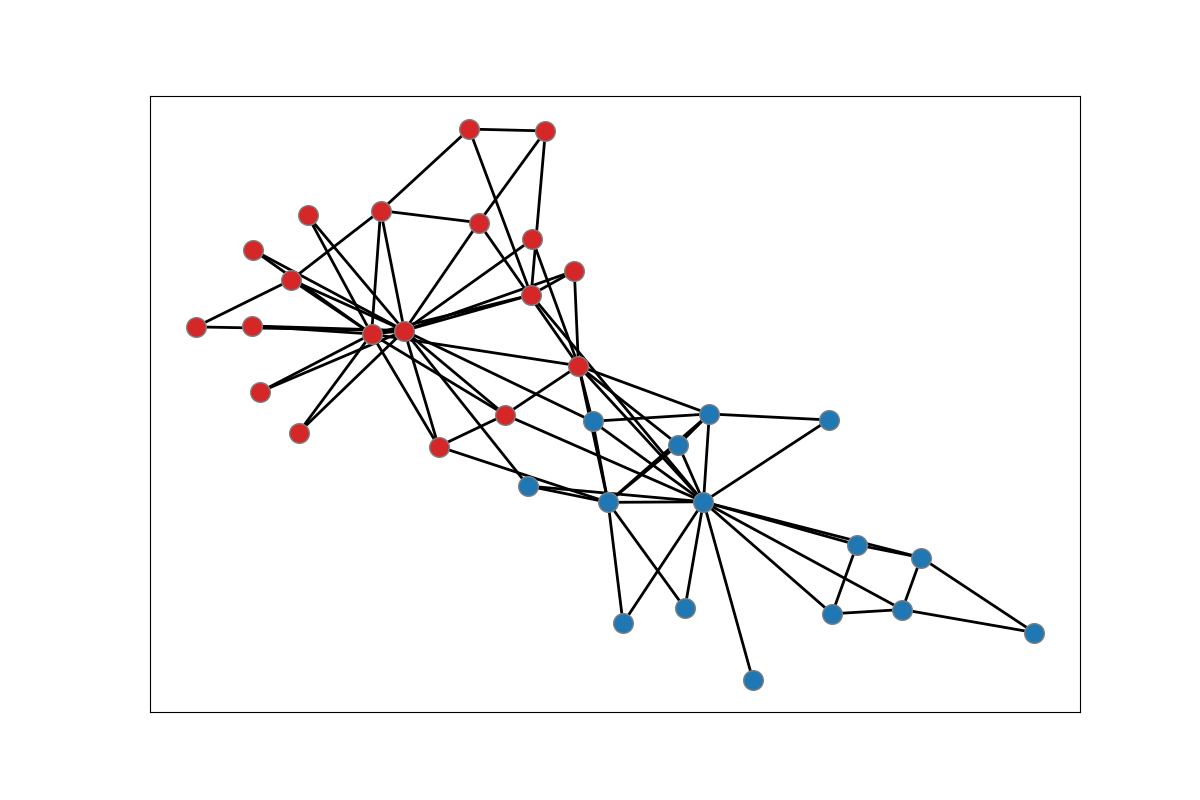}}
\subfigure[$T_{max}=1000$, threshold $1\%$ ]{\includegraphics[width=.45\linewidth]{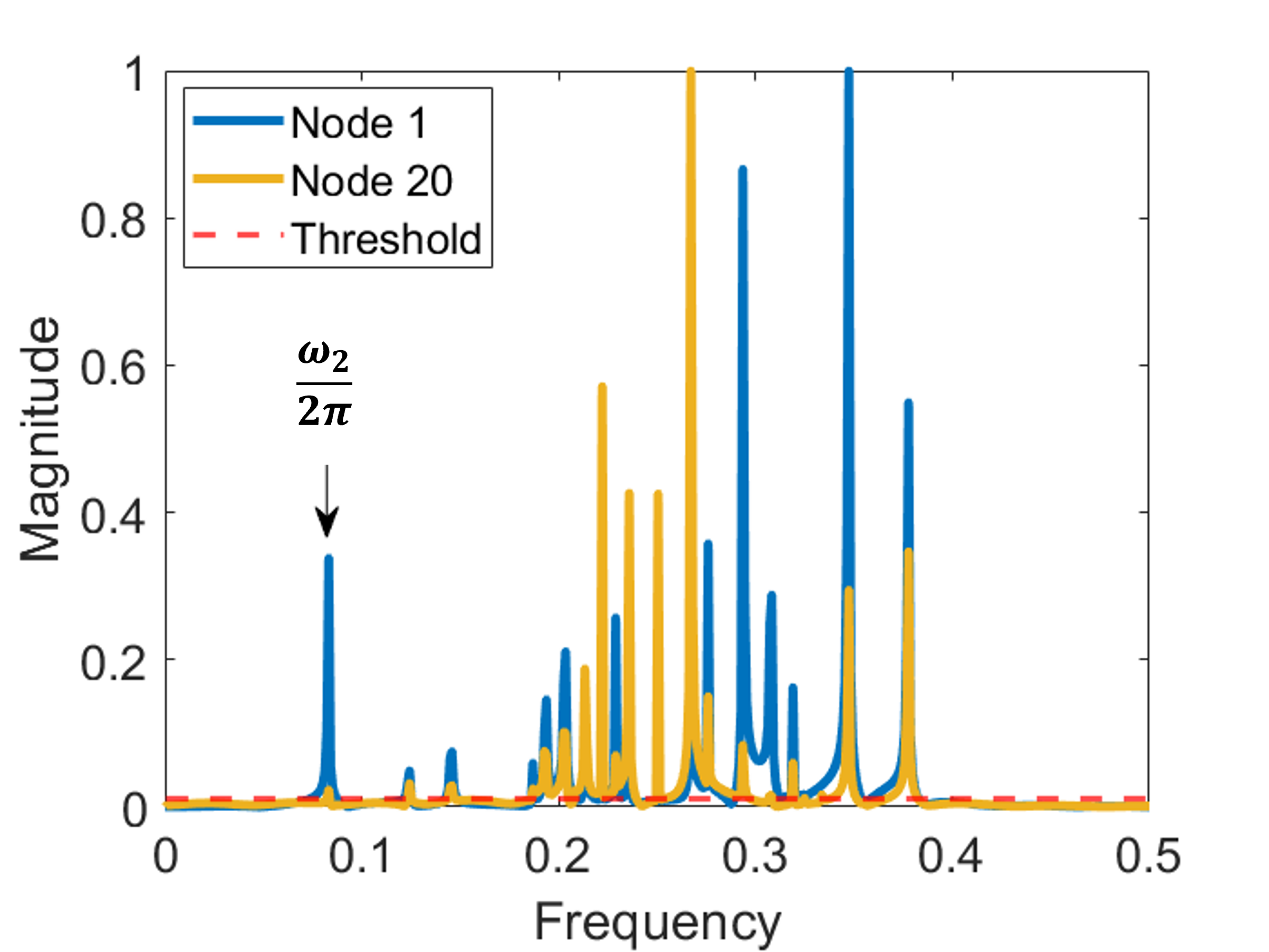}}
\caption{FFT of $[\mathbf{u}_i(1),...,\mathbf{u}_i(T_{max})]$ for Karate graph.}
\label{fig:karate_graph}
\end{figure}

\subsection{Dynamic Mode Decomposition for one-dimensional data}
Dynamic Mode Decomposition (DMD) is a powerful tool for analyzing the dynamics of nonlinear systems and can be interpreted as an approximation of Koopman modes \cite{tu2013dynamic, KNKWKSN18}. It is typically used as a centralized analysis approach on time series for all observation locations. Here, we apply DMD to one-dimensional time series data at a \emph{single} observation location by constructing the matrices $\mathbf{X}$ and $\mathbf{Y}$ for the exact DMD matrix $\mathbf{A}=\mathbf{Y}\mathbf{X}^+$ (where $^+$ denotes the pseudoinverse) using delay embeddings.

For any one-dimensional signal of the form
\begin{equation}\label{eq:1d_signal}
u(t) = \sum_{j=1}^J a_je^{i\omega_jt}
\end{equation}
where $\omega_j\in(-\pi, \pi), j=1,2,...,J$, are distinct frequencies,
the matrices $\mathbf{X}, \mathbf{Y} \in \mathbb{R}^{K\times M}$ where $K>1$ and $M>1$ can be defined by a time-delayed embedding of the signal,
\begin{align}\label{eq:Xmatrix}
\mathbf{X} &:=\begin{bmatrix}
u(0) & u(1) & ... & u(M-1) \\
u(1) & u(2) & ... & u(M) \\
\vdots & \vdots & \ddots & \vdots \\
u(K-1) & u(K) & ... & u(K+M-2)
\end{bmatrix}\\&=\begin{bmatrix}
\mathbf{x}(0) & \mathbf{x}(1) & \cdots & \mathbf{x}(M-1)
\end{bmatrix}, \nonumber
\end{align}
where $\mathbf{x}(t):=[u(t),u(t+1),...,u(t+K-1)]^T$,
\begin{equation}\label{eq:Ymatrix}
\mathbf{Y}:=\begin{bmatrix}
\mathbf{x}(1) & \mathbf{x}(2) & \cdots & \mathbf{x}(M)
\end{bmatrix}.
\end{equation}
Let $H > 1$, one can define a matrix  
\begin{align}\label{eq:Phi}
\boldsymbol\Phi_H &:=\begin{bmatrix}
1 & 1 & ... & 1 \\
e^{i\omega_1} & e^{i\omega_2} & ... & e^{i\omega_J} \\
\vdots & \vdots & \ddots & \vdots \\
e^{i(H-1)\omega_1} & e^{i(H-1)\omega_2} & ... & e^{i(H-1)\omega_J}
\end{bmatrix} \\
&= [\boldsymbol\phi_1, \boldsymbol\phi_2, \dots,\boldsymbol\phi_J], \nonumber
\end{align}
where $\boldsymbol\phi_j:=[1,e^{i\omega_j},...,e^{i(H-1)\omega_j}]^T$.

\begin{lemma}\label{lemma:delay_embedding_dmd}
For one-dimensional signal $u(t)$ defined by \eqref{eq:1d_signal}, if $K \geq J$ and $M \geq J$ of the matrices $\mathbf{X}$ and $\mathbf{Y}$ defined by \eqref{eq:Xmatrix} and \eqref{eq:Ymatrix} respectively, the eigenvalues of $\mathbf{A}=\mathbf{Y}\mathbf{X}^+$ are $\{e^{i\omega_j}\}_{j=1}^J$ and the columns of $\boldsymbol\Phi_{K}$ defined by \eqref{eq:Phi} are the corresponding eigenvectors. 
\end{lemma}

\begin{proof}
The matrix $\mathbf{X}$ can be decomposed as
\begin{equation}\label{eq:decomposition_X}
    \mathbf{X}(u) = \boldsymbol\Phi_{K} \diag(a_j) \boldsymbol\Phi_{M}^T,
\end{equation}
where $\diag(a_j)$ is a diagonal matrix with $a_j, j=1,2,\dots,J$, on the diagonal. Similarly, the matrix $\mathbf{Y}$ can be decomposed into
\begin{equation}\label{eq:decomposition_Y}
    \mathbf{Y}(u) = \boldsymbol\Phi_{K} \diag(a_j e^{i\omega_j}) \boldsymbol\Phi_{M}^T.
\end{equation}
The transpose of $\boldsymbol\Phi_{H}$ constructed from distinct $\omega_j$ is a Vandermonde matrix, and thus $\boldsymbol\Phi_{H}$ has the property that $\mrank(\boldsymbol\Phi_H)=\min(J, H)$ \cite{pan2020structure}. It follows that for $H\geq J$, $\boldsymbol\Phi_{H}$ has full column rank and $\boldsymbol\Phi_{H}^T$ has full row rank. Thus, if $K\geq J$, the matrix of the exact DMD
\begin{equation}\label{eq:A}
    \mathbf{A}=\mathbf{Y}\mathbf{X}^+ = \boldsymbol\Phi_{K} \diag(e^{i\omega_j})\boldsymbol\Phi_{K}^+,
\end{equation}
i.e., $\mathbf{A} \boldsymbol\Phi_{K} = \boldsymbol\Phi_{K} \diag(e^{i\omega_j})$. Thus, $\{e^{i\omega_j}\}_{j=1}^J$ are the eigenvalues of $\mathbf{A}$ and the columns of $\boldsymbol\Phi_{K}$ are the corresponding eigenvectors.
\end{proof}

\begin{prop}
 At node $l$, Dynamic Mode Decomposition on matrices $\mathbf{X}(u_l)$, $\mathbf{Y}(u_l)$ using local snapshots $\mathbf{u}_l(0),\mathbf{u}_l(1),...,\mathbf{u}_l(4N - 1)$ defined by \eqref{eq:u_discrete} with $K=M=2N$ yields exact eigenvalues of the Laplacian and the corresponding eigenvectors (scaled).
\end{prop}

\begin{proof}
By Proposition~\ref{lemma:wave_solution}, the number of modes $J\leq2N$ where $N$ is the number of graph nodes. By Proposition~\ref{lemma:delay_embedding_dmd}, the eigenvalues of $\mathbf{A(\mathbf{u}_l)}=\mathbf{Y(\mathbf{u}_l)}\mathbf{X(\mathbf{u}_l)}^+$ are $\{e^{i\omega_j}\}_{j=1}^J$ and the columns of $\boldsymbol\Phi_{K}$ defined by \eqref{eq:Phi} are the corresponding eigenvectors. The eigenvalues of the Laplacian can be derived from \eqref{eq:eigenvalues_M} and \eqref{eq:alpha} and given by $\lambda_j = (2 - e^{i\omega_j} - e^{-i\omega_j})/c^2$.
From \eqref{eq:Xmatrix} and \eqref{eq:decomposition_X}
\begin{equation}
\mathbf{x}(t) = \boldsymbol\Phi_{K} \diag(e^{i\omega_j t})\mathbf{a}, \nonumber
\end{equation}
where $\mathbf{a} := [a_1, a_2,\dots,a_J]^T$. Thus, $\mathbf{a}$ can be computed by solving the system of linear equations $\mathbf{x}(0)=\boldsymbol\Phi_{K}\mathbf{a}$.
From \eqref{eq:u_discrete}, the coefficients $\{a_j^{(l)}\}_{l=1}^{N}$ corresponding to $e^{i\omega_j}$ are proportional to $\mathbf{v}^{(j)}$
by a factor $p_j \mathbf{u}(0)^T \mathbf{v}^{(j)}$. 
\end{proof}

Note that numerically, the eigenvectors are usually normalized, i.e., $\hat{\boldsymbol\Phi}_{K} = [\hat{\phi}_1, \hat{\phi}_2,\dots,\hat{\phi}_J]$, $\hat{\phi}_j = \phi_j/\|\phi_j\|_2$. Since $\phi_{j,1}=1$ implies $\hat{\phi}_{j,1}=1/\|\phi_j\|_2$, one can solve for $\hat{\mathbf{a}}$ using $\mathbf{x}(0)=\hat{\boldsymbol\Phi}_{K} \hat{\mathbf{a}}$ and $a_j=\hat{\phi}_{j,1}\hat{a}_j$.

The method to compute coefficients $\mathbf{a}$ is shown in Algorithm~\ref{alg:dmd} and decentralized clustering in Algorithm~\ref{alg:clustering}. 

\begin{algorithm}
\begin{algorithmic}[1]
\State Compute (reduced) SVD of $\mathbf{X}$, i.e., $\mathbf{X}=\mathbf{U}\mathbf{\Sigma}\mathbf{V}^*$.
\State Define the matrix $\tilde{\mathbf{A}} \triangleq \mathbf{U}^*\mathbf{Y}\mathbf{V}\mathbf{\Sigma}^{-1}$.
\State Compute eigenvalues $\mu$ and eigenvectors $w$ of $\tilde{\mathbf{A}}$, i.e., $\tilde{\mathbf{A}}w = \mu w$. Each nonzero eigenvalue $\mu$ is a DMD eigenvalue.
\State The DMD mode corresponding to $\mu$ is then given by $\hat{\boldsymbol\phi}\triangleq\frac{1}{\mu}\mathbf{Y}\mathbf{V}\mathbf{\Sigma}^{-1}w$.
\State Compute $\hat{\mathbf{a}}$ by solving the linear system $\hat{\boldsymbol\Phi} \hat{\mathbf{a}}=\mathbf{x}(0)$, where the columns of $\hat{\boldsymbol\Phi}$ are the eigenvectors sorted in decreasing order based on the the real part of the eigenvalues.
\State $a_i^{(j)}=\hat{\boldsymbol\phi}_{j,1}\hat{a}_j, j=1,2,...,k$.
\end{algorithmic}
\caption{DMD$(\mathbf{X}, \mathbf{Y})$: For computing eigenvalues and eigenvector components at node $i$. }
\label{alg:dmd}
\end{algorithm}

\begin{algorithm}
\begin{algorithmic}[1]
\State $\mathbf{u}_i(0)\leftarrow \text{Random}([0,1])$
\State $\mathbf{u}_i(-1) \leftarrow \mathbf{u}_i(0)$
\State $t \leftarrow 1$
\While{$t < T_{max}$} 
    \State $\mathbf{u}_i(t) \leftarrow 2\mathbf{u}_i(t - 1) - \mathbf{u}_i(t - 2) - c^2 \sum\limits_{\mathclap{j\in\mathcal{N}(i)}} \mathcal{L}_{ij}\mathbf{u}_j(t-1)$
    \State $t \leftarrow t+1$
\EndWhile 
\State Create the matrices $\mathbf{X}_i, \mathbf{Y}_i \in \mathbb{R}^{K\times M} $ defined by \eqref{eq:Xmatrix} and \eqref{eq:Ymatrix} at Node $i$ using $\mathbf{u}_i(0),\mathbf{u}_i(1),...,\mathbf{u}_i(T_{max}-1)$ where $K+M=T_{max}$
\State $\mathbf{v}_i \leftarrow a_i$ from DMD$(\mathbf{X}_i,\mathbf{Y}_i)$ by Algorithm~\ref{alg:dmd} 
\For{$j\in\{1,\cdots,k\}$}
	\If{$\mathbf{v}_i^{(j)}>0$}
		\State $A_j\leftarrow 1$
	\Else
		\State $A_j\leftarrow 0$
	\EndIf
\EndFor
\State $\text{ClusterNumber}\leftarrow\sum\limits_{j=1}^{k}A_j2^{j-1}$
\end{algorithmic}
\caption{Wave equation based graph clustering}
\label{alg:clustering}
\end{algorithm}

\section{Numerical results}\label{sec:numerics}
\subsection{A line graph}

In this example, we consider the simple line graph shown in Figure~\ref{fig:line_graph} comprising $N=50$ nodes. Edge $(25,26)$ has weight 1 while all other edges have weight 5. Both centralized spectral clustering and the wave equation based decentralized clustering predict two clusters cut at the weakly connected edge $(25,26)$. The elements of the 2\textsuperscript{nd} Laplacian eigenvector $\mathbf{v}^{(2)}$ are shown in Figure~\ref{fig:line_clustering}(a) and the coefficients derived from DMD are shown in Figure~\ref{fig:line_clustering}(b). The matrices $\mathbf{X}$ and $\mathbf{Y}$ of size $100\times100$ are constructed at each node by local snapshots $\mathbf{u}_i(0),\mathbf{u}_i(1),...,\mathbf{u}_i(199)$. The proposed method predicts the same clustering as the spectral method. DMD coefficients are proportional (here the constant of proportionality can be negative) to the second eigenvector of Laplacian $\mathbf{v}^{(j)}$, thus, the clustering is the same. Note that, for the line graph, the eigenvalues of $\mathbf{M}$ are computed perfectly by our proposed approach, when compared to closed form solutions. The eigenvalues and the first 5 eigenvectors are shown in Figure~\ref{fig:line_dmd_eigen}. 

\begin{figure}[h]
\centering
\includegraphics[width=.9\linewidth]{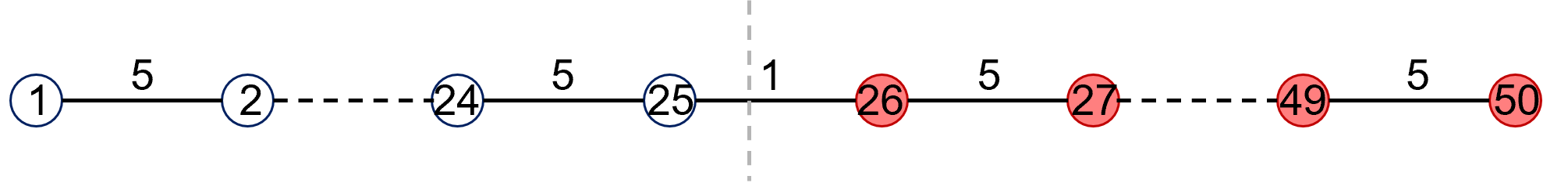}
\caption{Spectral clustering of a line graph with a weak connection.}
\label{fig:line_graph}
\end{figure}

\begin{figure}[h]
\centering
\subfigure[Laplacian 2\textsuperscript{nd} eigenvector]{\includegraphics[width=.4\linewidth]{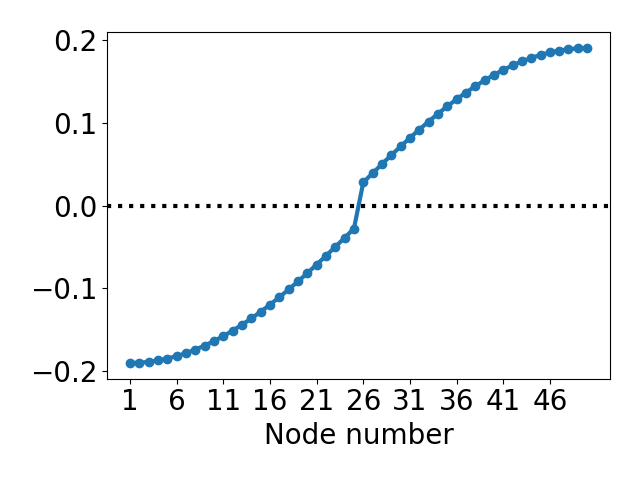}}
\subfigure[DMD coefficients of $e^{i\omega_2}$]{\includegraphics[width=.4\linewidth]{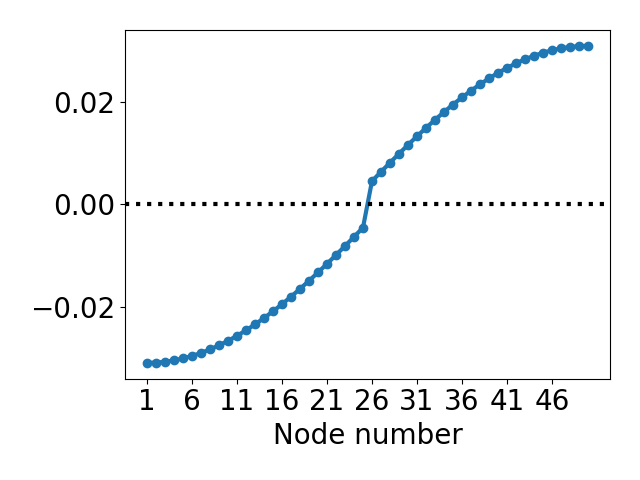}}
\caption{Clustering of the line graph shown in Figure~\ref{fig:line_graph}. The signs indicate clustering.}
\label{fig:line_clustering}
\end{figure}

\begin{figure}[h]
\centering
\subfigure[Eigenvalues]{\includegraphics[width=.33\linewidth]{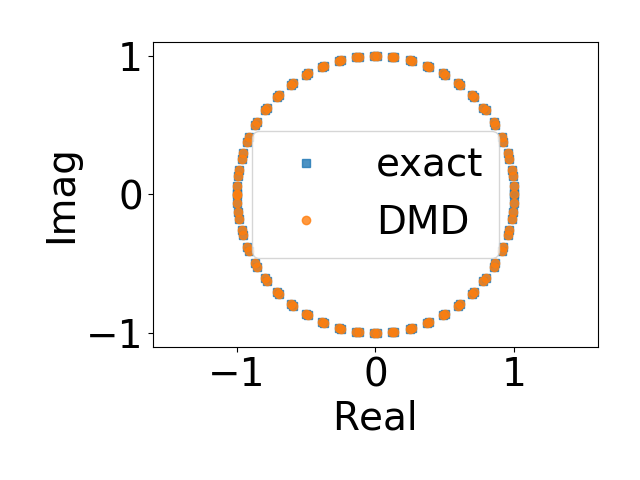}}
\subfigure[Modes: $\Re(\hat{\phi}_j)$]{\includegraphics[width=.32\linewidth]{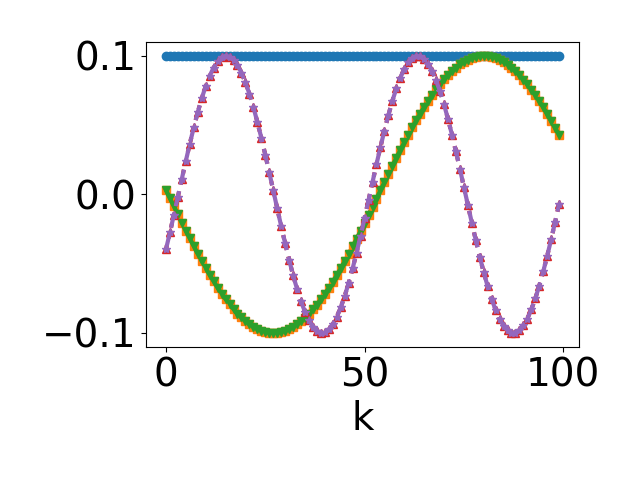}}
\subfigure[Modes: $\Im(\hat{\phi}_j)$]{\includegraphics[width=.32\linewidth]{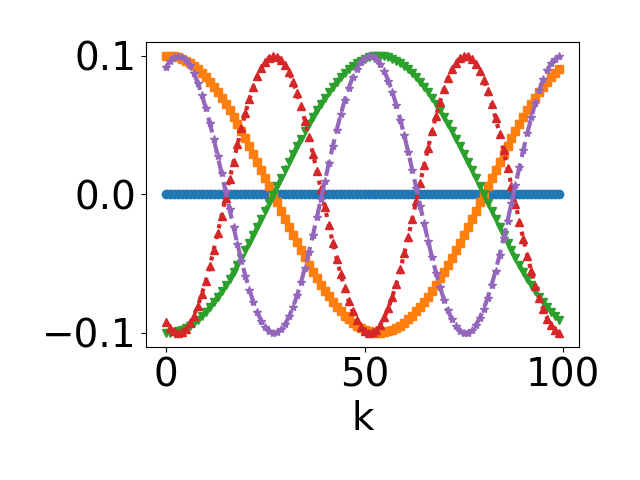}}
\caption{DMD eigenvalues and modes $\hat{\phi}_j$ at Node 1 for wave propagation on the line graph, j=0 (blue), 1 (orange), 2 (green), 3 (red) and 4 (purple). The exact eigenvalues are the eigenvalues of the matrix $\mathbf{M}$ defined in \eqref{eq:matrix_wave}.}
\label{fig:line_dmd_eigen}
\end{figure}

\subsection{Zachary Karate club graph}
We now demonstrate our algorithm on the Zachary Karate club graph. The entries of the 2\textsuperscript{nd} Laplacian eigenvector $\mathbf{v}^{(2)}$ are shown in Figure~\ref{fig:karate_clustering_v2}(a) and the coefficients derived from DMD are shown in Figure~\ref{fig:karate_clustering_v2}(b). Instead of $2N\times 2N$, we construct reduced matrices $\mathbf{X}$ and $\mathbf{Y}$ of size $25\times25$ by using local snapshots $\mathbf{u}_i(0),\mathbf{u}_i(1),...,\mathbf{u}_i(49)$. The eigenvalues and the first 5 eigenvectors are shown in Figure~\ref{fig:karate_dmd_eigen}. With reduced snapshots, we can still compute the dominant eigenvalues and predict the same clustering as the spectral method. Currently, the number of reduced snapshots and matrix size are determined empirically. In Figure~\ref{fig:karate_clustering_v2} (and subsequent figures), the ``exact'' eigenvalues refer to the direct computation of eigenvalues of $\mathbf{M}$ for comparison.

\begin{figure}[h]
\centering
\subfigure[Laplacian 2\textsuperscript{nd} eigenvector]{\includegraphics[width=.4\linewidth]{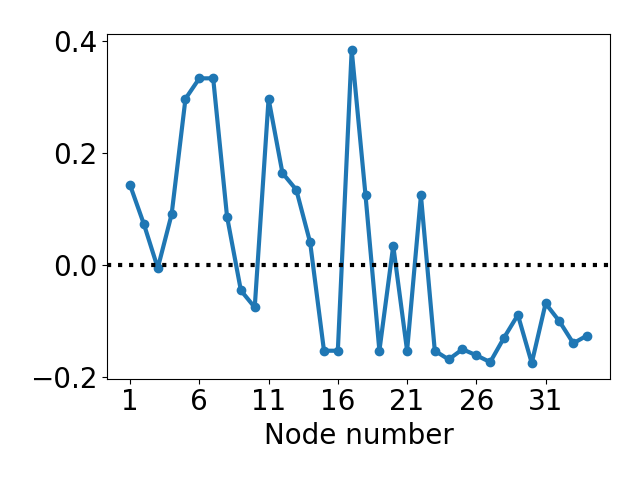}}
\subfigure[DMD$(\mathbf{X},\mathbf{Y})$]{\includegraphics[width=.4\linewidth]{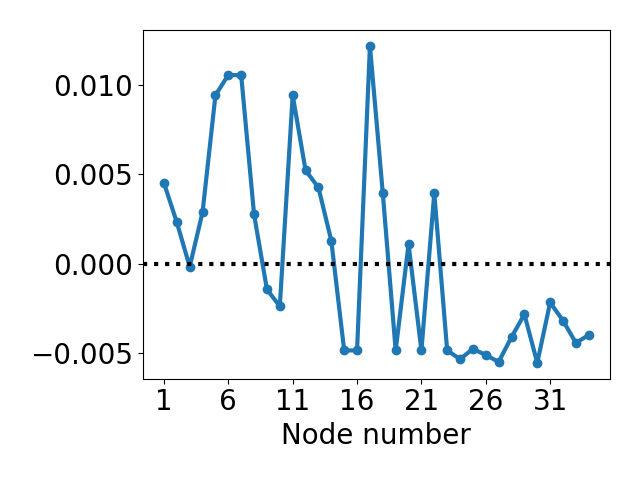}}
\caption{The second eigenvector of the Karate graph shown in Figure~\ref{fig:karate_graph}. The signs of the eigenvector entries indicate assignment of the corresponding node into one of two clusters.}
\label{fig:karate_clustering_v2}
\end{figure}

\begin{figure}[h]
\centering
\subfigure[Eigenvalues]{\includegraphics[width=.33\linewidth]{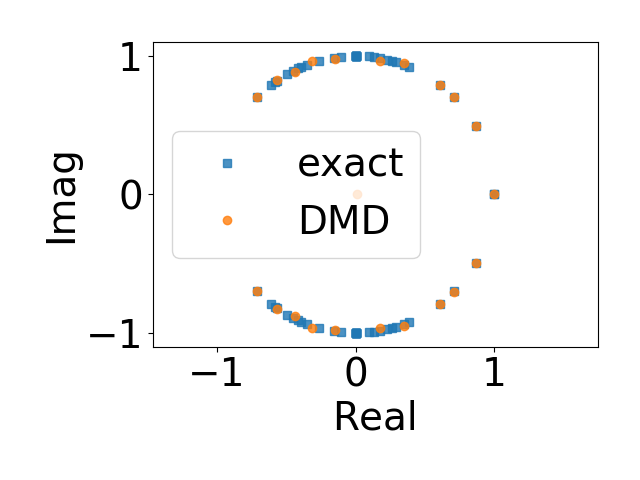}}
\subfigure[Modes: $\Re(\hat{\phi}_j)$]{\includegraphics[width=.32\linewidth]{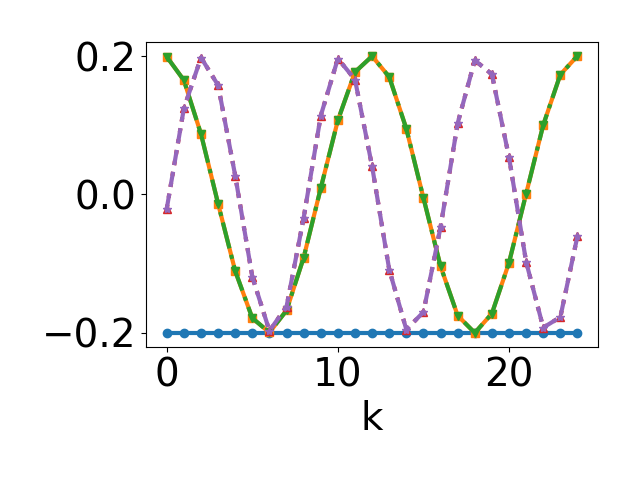}}
\subfigure[Modes: $\Im(\hat{\phi}_j)$]{\includegraphics[width=.32\linewidth]{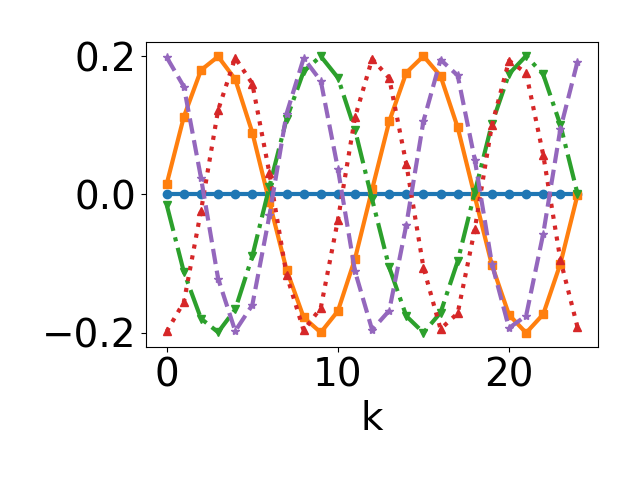}}
\caption{DMD eigenvalues and modes $\hat{\phi}_j$ at Node 1 for wave propagation on the Karate graph, j=0 (blue), 1 (orange), 2 (green), 3 (red) and 4 (purple). The exact eigenvalues are the eigenvalues of $\mathbf{M}$ in \eqref{eq:matrix_wave}.}
\label{fig:karate_dmd_eigen}
\end{figure}

\subsection{A synthetic network}

In this example, we demonstrate our algorithm on a synthetic network,  with 400 nodes and 2335 edges (weights varying between 0 and 2) that has 4 clusters as shown in Fig.~\ref{fig:synthetic_4clusters}. We construct reduced matrices $\mathbf{X}$ and $\mathbf{Y}$ of size $70\times430$ as the input for DMD. The spectral gap between $\omega_4$ and $\omega_5$, shown in Fig.~\ref{fig:eigenvalues_synthetic}, indicates that the number of clusters is four. Our algorithm gives the same clustering result as the standard spectral clustering method as shown in Fig.~\ref{fig:eigenvalues_synthetic}(b). Note that two nodes with eigenvector entries close to 0 (circled in Fig.~\ref{fig:synthetic_clustering_v2}(a) and Fig.~\ref{fig:synthetic_clustering_v3}(a)) are assigned to incorrect clusters. Instead of following the steps in lines 10--17 in Algorithm~\ref{alg:clustering}, one can cluster the points $\mathbf{v}_i\in \mathbb{R}^k, i=1,2,...,N$ in line 9, into $k$ clusters using the distributed $k$-means algorithm~\cite{balcan2013distributed}. Note that if we let $\mathbf{V}$ be the matrix containing the eigenvectors $\mathbf{v}^{(j)}, j=1,2,...,k$, as columns, then $\mathbf{v}_i$ is the $i$-th row of $\mathbf{V}$. The correct clustering result is recovered using $k$-means with $k=4$ for this example as shown in Fig.~\ref{fig:synthetic_4clusters}.

\begin{figure}
\centering
\includegraphics[width=.5\linewidth]{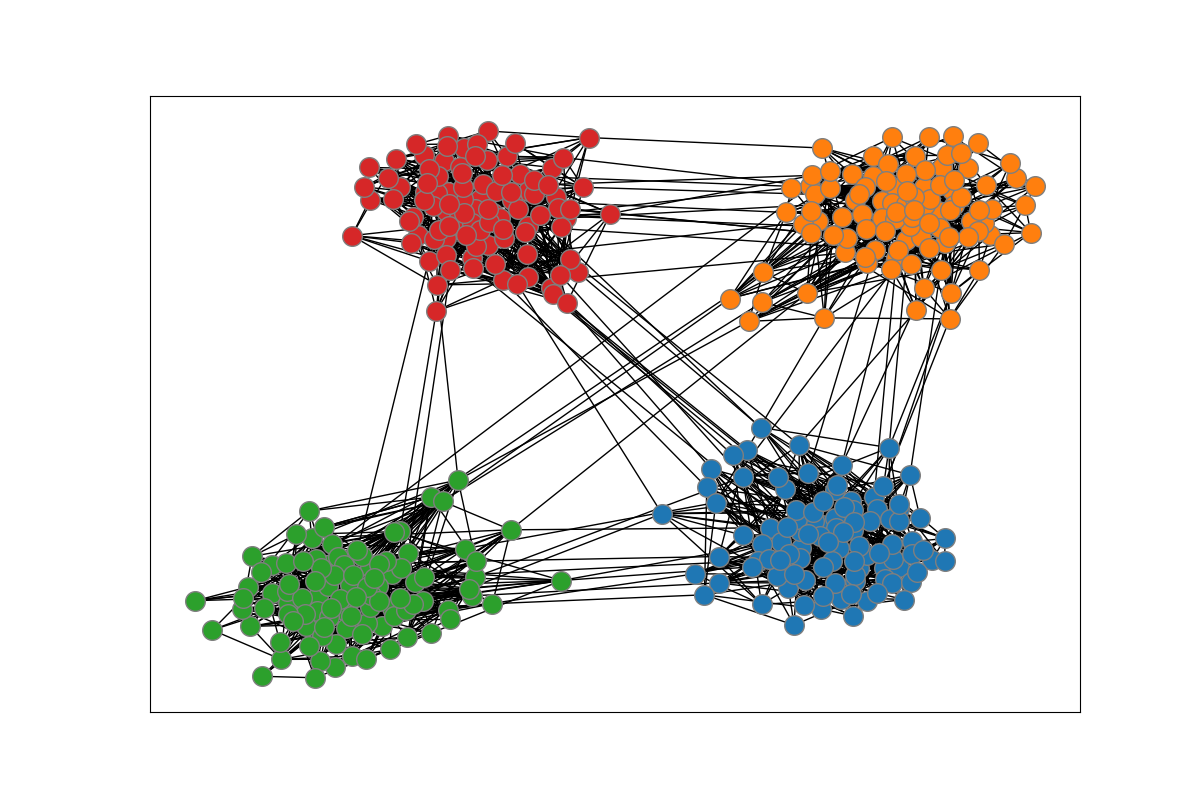}
\caption{A synthetic network with 4 clusters.}
\label{fig:synthetic_4clusters}
\end{figure}

\begin{figure}
\centering
\subfigure[Laplacian eigenvalues]{\includegraphics[width=.4\linewidth]{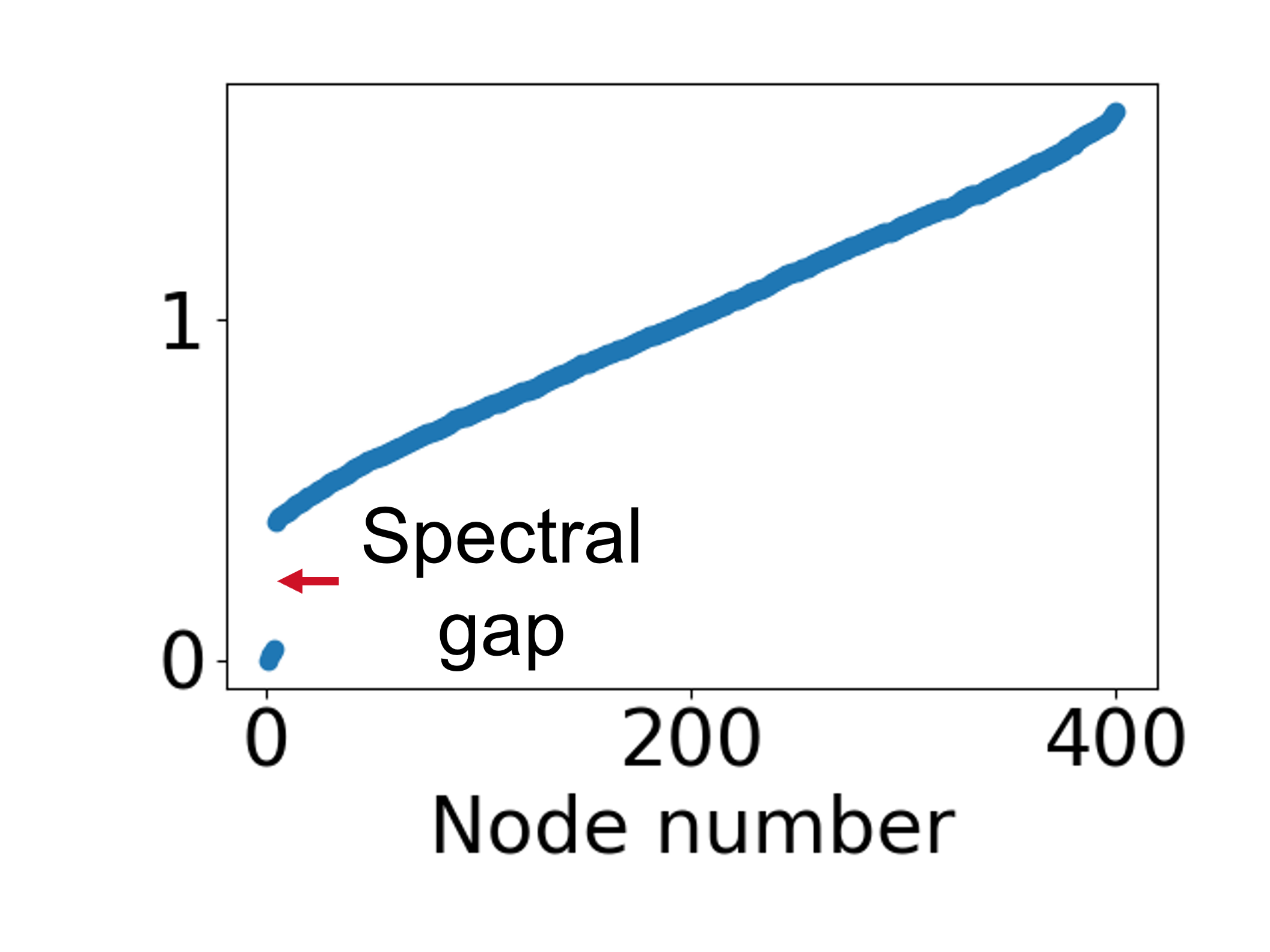}}
\subfigure[DMD eigenvalues]{\includegraphics[width=.4\linewidth]{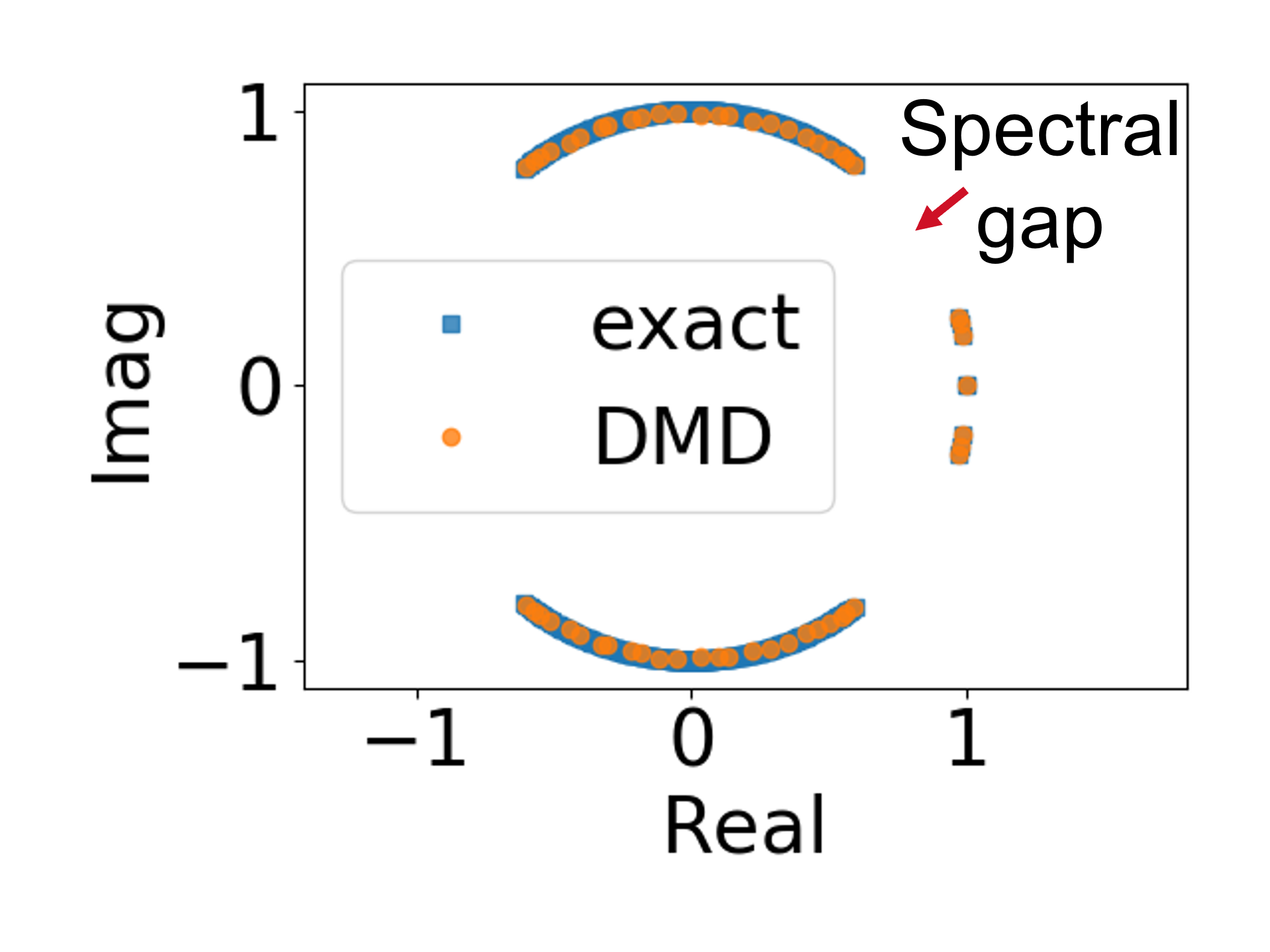}}
\caption{Laplacian and DMD eigenvalues for the synthetic example. The spectral gap between $\omega_4$ and $\omega_5$ indicates the number of clusters to be four. The exact eigenvalues are the eigenvalues of the matrix $\mathbf{M}$ defined in \eqref{eq:matrix_wave}.}
\label{fig:eigenvalues_synthetic}
\end{figure}

\begin{figure}[h]
\centering
\subfigure[Laplacian 2\textsuperscript{nd} eigenvector]{\includegraphics[width=.4\linewidth]{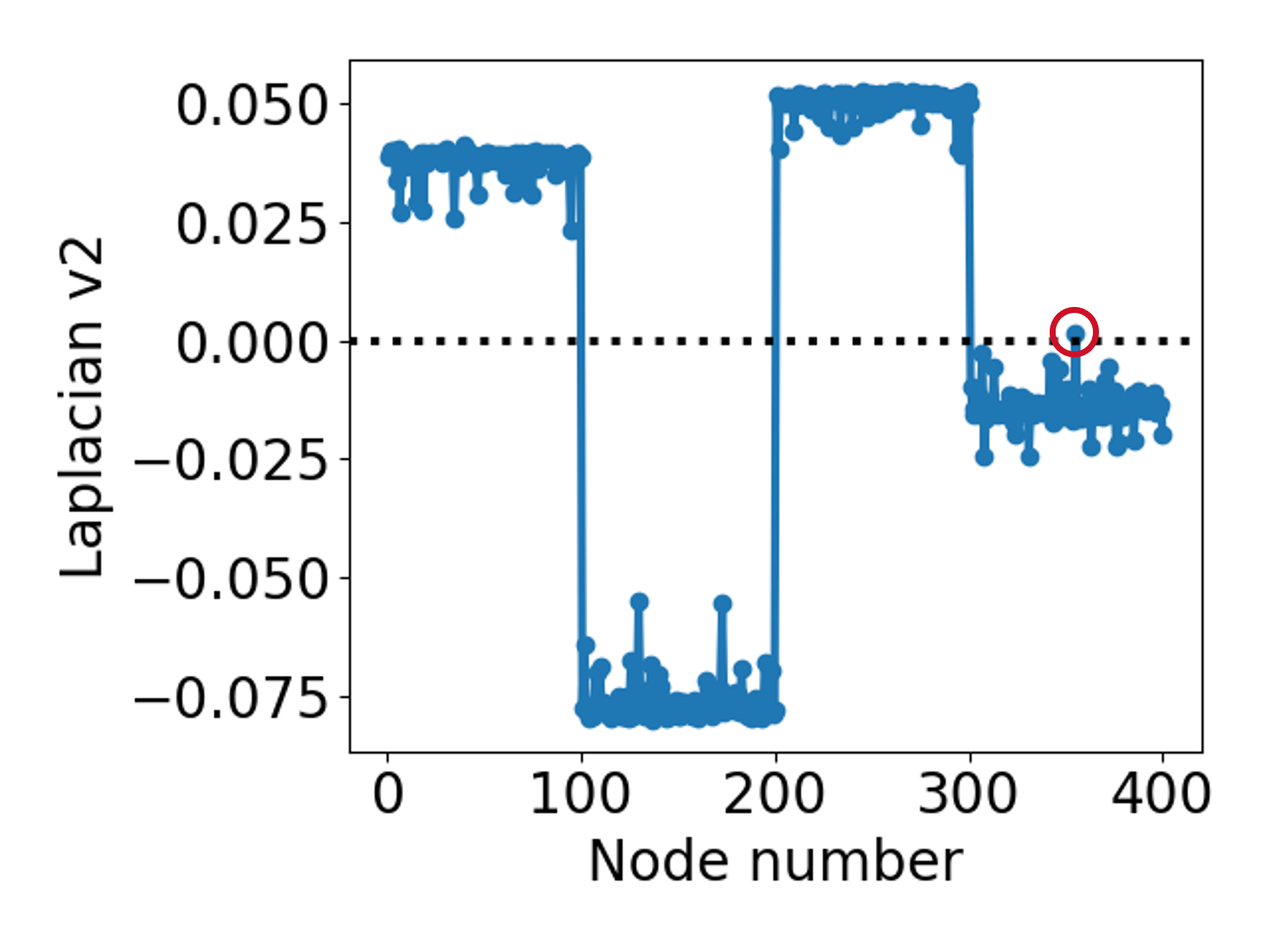}}
\subfigure[DMD$(\mathbf{X},\mathbf{Y})$]{\includegraphics[width=.4\linewidth]{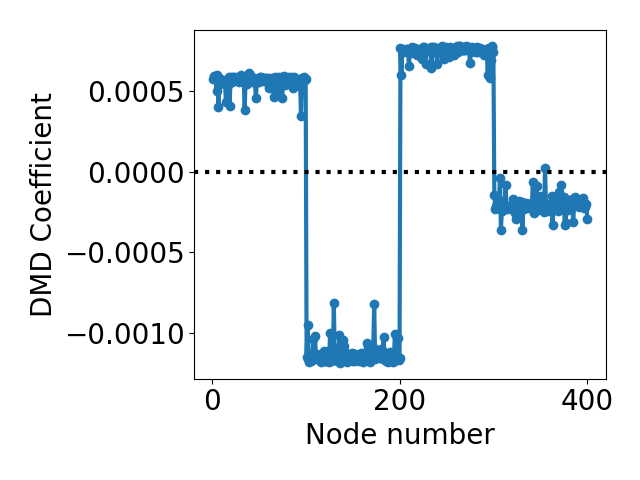}}
\caption{The second eigenvector of the synthetic graph. The signs of the eigenvector entries indicate assignment of the corresponding node to one of two clusters.}
\label{fig:synthetic_clustering_v2}
\end{figure}

\begin{figure}[h]
\centering
\subfigure[Laplacian 3\textsuperscript{rd} eigenvector]{\includegraphics[width=.4\linewidth]{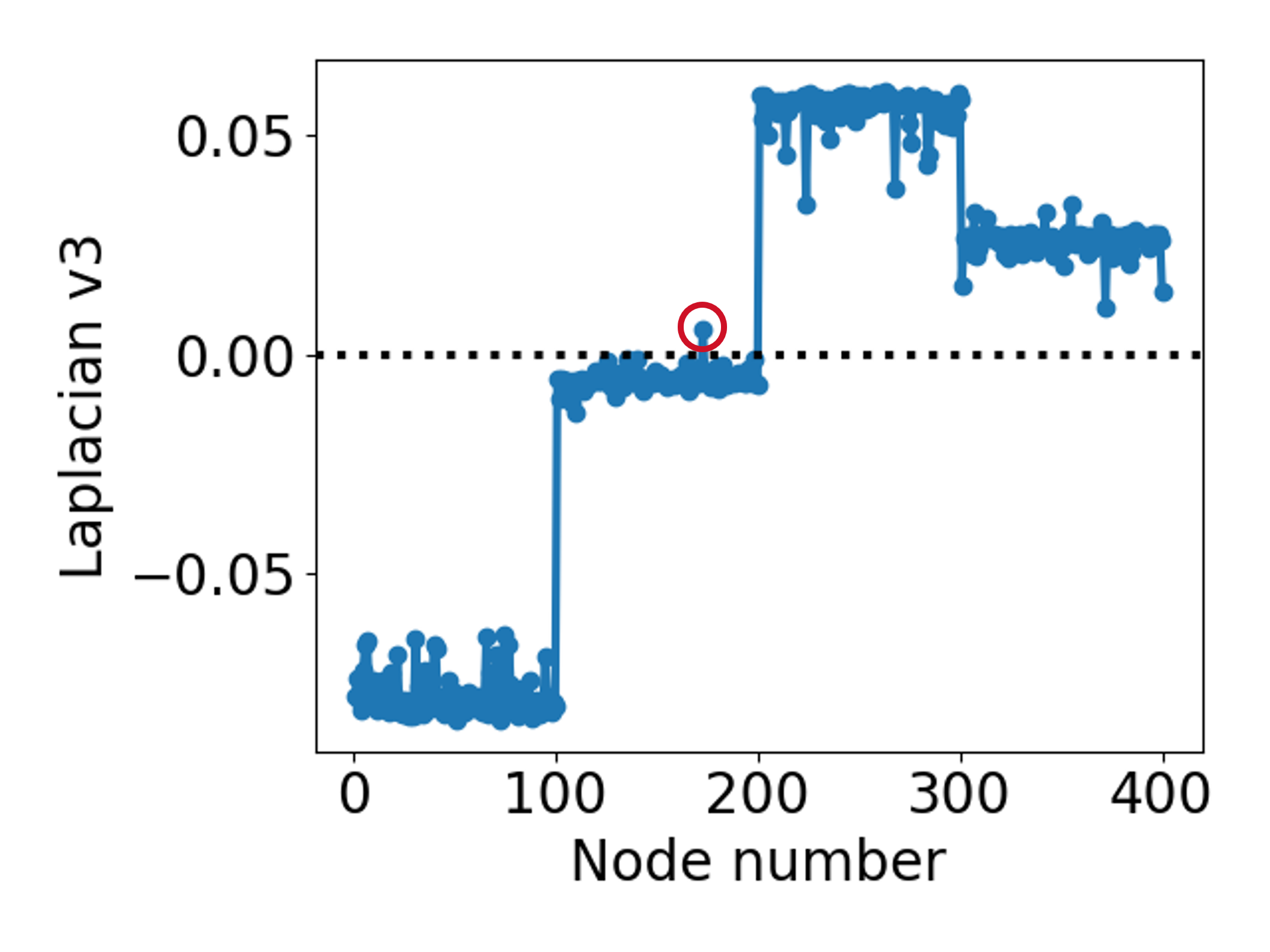}}
\subfigure[DMD$(\mathbf{X},\mathbf{Y})$]{\includegraphics[width=.4\linewidth]{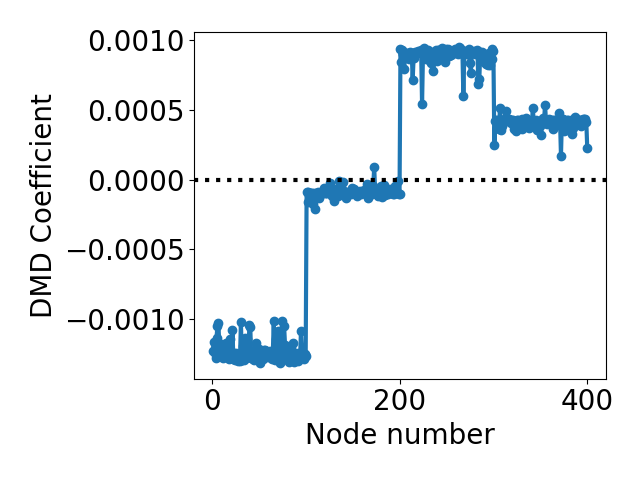}}
\caption{The third eigenvector of the synthetic graph. The signs of the eigenvector entries combined with those of the second eigenvector indicate assignment of the corresponding node to one of four clusters}
\label{fig:synthetic_clustering_v3}
\end{figure}

\subsection{A Facebook network}

In this example, we demonstrate our algorithm on ego networks in publicly available Facebook data~\cite{snapnets} with 4,039 nodes and 88,234 edges. We construct reduced matrices $\mathbf{X}$ and $\mathbf{Y}$ with size $600\times5400$ as the input for DMD. The spectral gap between $\omega_7$ and $\omega_8$, as shown in Fig.~\ref{fig:facebook}(a), indicates that the number of clusters is seven. Using our decentralized DMD based wave equation clustering approach with $k$-means gives the same clustering result as the spectral clustering method as shown in Fig.~\ref{fig:facebook}(b), thereby demonstrating the utility of our proposed approach on computing clusters in real-world networks.

\begin{figure}
\centering
\subfigure[Eigenvalues]{\includegraphics[width=.4\linewidth]{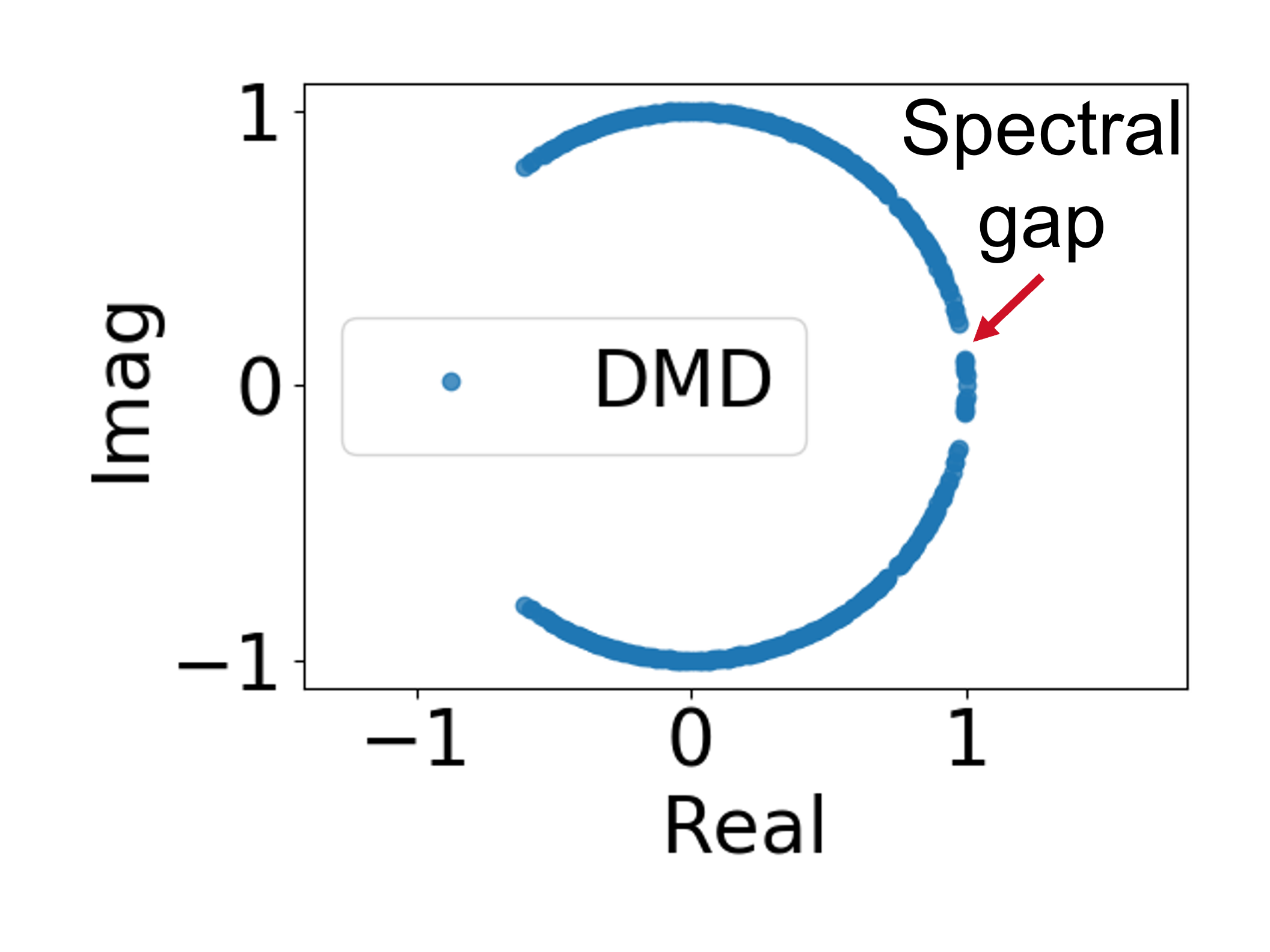}}
\subfigure[Clustering]{\includegraphics[width=.5\linewidth]{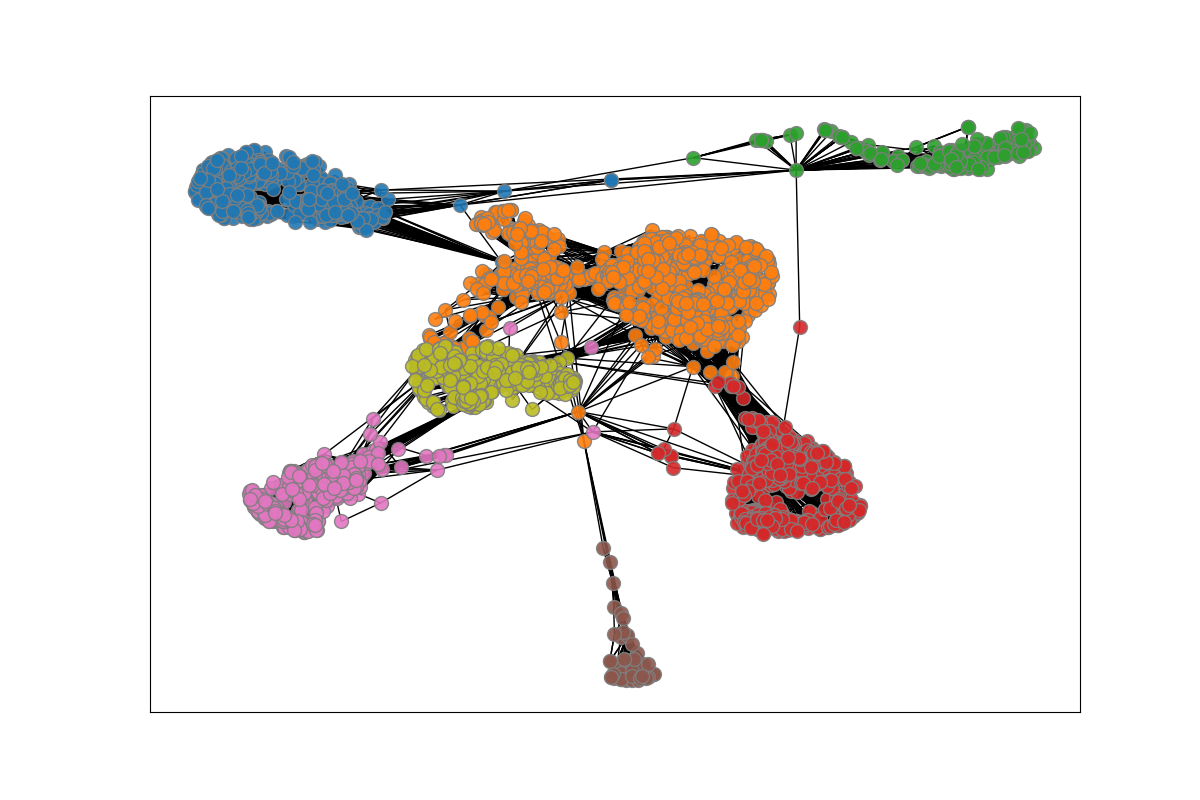}}
\caption{Clustering of a Facebook network. The spectral gap between $\omega_7$ and $\omega_8$ indicates that the number of clusters is $7$.}
\label{fig:facebook}
\end{figure}

\section{Comparison with FFT based wave equation clustering}\label{sec:comparison}
In this section, we compare the performance of wave equation clustering using DMD with the existing FFT version. In particular, we compare the minimum $T_{max}$ (number of wave equation updates) for DMD and FFT required to compute the same assignment as centralized spectral clustering as shown in Table~\ref{tbl:Tmax}. For FFT, a threshold $1\%$ in magnitude is used to identify the frequencies and the corresponding coefficients. DMD based clustering converges over 20 times faster when compared to the existing FFT based approach.

Fig.~\ref{fig:freq_error} shows the mean relative error in the lowest frequency $\omega_2$ with $T_{max}\in \{64, 128, 256, 512, 1024 \}$ at all the nodes for the Zachary Karate club graph.
Note that the resolution of frequencies computed by FFT is $2/T_{max}$. Sufficiently large $T_{max}$ guarantees the accuracy of computed FFT frequencies assuming an appropriate threshold. As evident from the plot, our DMD approach provides \emph{orders of magnitude} improvement over existing FFT based schemes.

\begin{figure}[h]
\centering
\includegraphics[width=.45\linewidth]{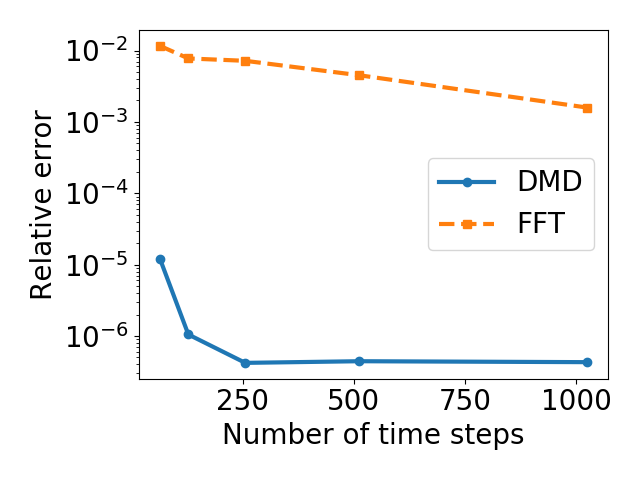}
\caption{Mean relative error of the lowest frequency averaged over the nodes, $\frac{1}{N}\sum_{k=1}^N|\hat{\omega}_2^{(k)} - \omega_2|/\omega_2$, where $\hat{\omega}_2^{(k)}$ is the FFT computed value at Node~$k$ for the Zachary Karate club example.}
\label{fig:freq_error}
\end{figure}

\begin{table}
\caption{Least $T_{max}$ required to recover the same assignment as spectral clustering.}
\centering
\begin{tabular}{ c c c c}
 \hline
 Algorithm &  Karate & Synthetic  & Facebook\\
 \hline
 FFT & 1024$^a$ (1000$^b$) & $>16,384^a$ (16,000$^b$) & $>131,072^a$ \\  
 DMD & 50 & 500 & 6000 \\
 \hline
\end{tabular}\\[1ex]
$^a$ Keep $T_{max}$ to be a power of 2. 
$^b$ Increase by 100.
\label{tbl:Tmax}
\end{table}

\section{Conclusions}\label{sec:conclusion}
In this work, we have developed a novel algorithm for computing balanced clusters in undirected graphs by exploiting dynamics on graphs~\cite{sahai2020dynamical}. In particular, we developed a new decentralized graph clustering method based on the evolution of the wave equation followed by a local dynamic mode decomposition (DMD) step. At each node, one can compute cluster assignments by exchanging information with nearest neighbors. The proposed method addresses computational issues associated with existing fast Fourier transform (FFT) based approaches. The decentralized algorithm is demonstrated on a line graph where all the eigenvalues can be computed exactly.  We also presented results on a synthetic example (with known number of clusters) and a real world social network. In all cases, we demonstrated that the clustering predicted by the decentralized approach is the same as standard spectral techniques.

We observed that the dominant eigenvalues can be computed using a reduced number of snapshots if the number of clusters is low, thus requiring the first few eigenvectors and eigenvalues. 
In practice, one can start with a small number of snapshots and then increase it till the dominant frequencies converged.
In particular, we demonstrated that the DMD approach reduces the relative error and required number of wave equation time steps by orders of magnitude. This improvement becomes more pronounced with the size of the graph.

We note that the accurate estimation of the spectral properties of a graph has applications beyond cluster computation. For instance, the spectral properties of a graph are related to its diameter~\cite{chung1997spectral}, pagerank~\cite{page1999pagerank}, and isomorphism testing~\cite{klus2018spectral} to name a few. In future work, we plan to extend our Koopman framework to a range of distributed graph applications that require the computation of spectral properties of a graph. 

\bibliographystyle{IEEEtran}

\end{document}